\documentclass[conference]{IEEEtran}
\IEEEoverridecommandlockouts

\usepackage{cite}
\usepackage{amsmath,amssymb,amsfonts,amsthm}
\usepackage{algorithm}
\usepackage{algorithmic}
\usepackage{graphicx}
\usepackage{booktabs}
\usepackage{multirow}
\usepackage{textcomp}
\usepackage{xcolor}
\usepackage{tikz}
\usetikzlibrary{shapes.geometric, shapes.misc, arrows.meta, calc, fit, positioning}
\usepackage[hidelinks]{hyperref} 

\usepackage[left=1.62cm,right=1.62cm,top=1.8cm,bottom=2.4cm]{geometry}
\usepackage{flushend}

\def\BibTeX{{\rm B\kern-.05em{\sc i\kern-.025em b}\kern-.08em
    T\kern-.1667em\lower.7ex\hbox{E}\kern-.125emX}}

\newtheorem{theorem}{Theorem}

\theoremstyle{definition}
\newtheorem{definition}{Definition}

\begin{document}

\title{A Bayesian Incentive Mechanism for Poison-Resilient Federated Learning}

\author{
  \IEEEauthorblockN{Daniel Commey\IEEEauthorrefmark{1},
                    Rebecca A. Sarpong\IEEEauthorrefmark{2},
                    Griffith S. Klogo\IEEEauthorrefmark{3},
                    Winful Bagyl-Bac\IEEEauthorrefmark{4}, and
                    Garth V. Crosby\IEEEauthorrefmark{5}}
  \IEEEauthorblockA{\IEEEauthorrefmark{1}Department of Multidisciplinary Engineering,
                    Texas A\&M University, College Station, TX, USA\\
                    dcommey@tamu.edu}
  \IEEEauthorblockA{\IEEEauthorrefmark{2}Department of Statistics and Actuarial Science,
                    KNUST, Kumasi, Ghana\\
                    rasarpong5@st.knust.edu.gh}
  \IEEEauthorblockA{\IEEEauthorrefmark{3}Department of Computer Engineering,
                    KNUST, Kumasi, Ghana\\
                    gsklogo.coe@knust.edu.gh}
  \IEEEauthorblockA{\IEEEauthorrefmark{4}Department of Computer Science, 
                    George Washington University, Washington, DC, USA\\
                    winful.bagylbac@gwu.edu}
  \IEEEauthorblockA{\IEEEauthorrefmark{5}Department of Engineering Technology and Industrial Distribution,
                    Texas A\&M University, College Station, TX, USA\\
                    gvcrosby@tamu.edu}
}

\maketitle

\begin{abstract}
Federated learning (FL) enables collaborative model training across decentralized clients while preserving data privacy. However, its open-participation nature exposes it to data-poisoning attacks, in which malicious actors submit corrupted model updates to degrade the global model. Existing defenses are often \emph{reactive}, relying on statistical aggregation rules that can be computationally expensive and that typically assume an honest majority. This paper introduces a \emph{proactive}, economic defense: a lightweight Bayesian incentive mechanism that makes malicious behavior economically irrational. Each training round is modeled as a Bayesian game of incomplete information in which the server, acting as the principal, uses a small, private validation dataset to verify update quality before issuing payments. The design satisfies Individual Rationality (IR) for benevolent clients, ensuring their participation is profitable, and Incentive Compatibility (IC), making poisoning an economically dominated strategy. Extensive experiments on non-IID partitions of MNIST and FashionMNIST demonstrate robustness: with 50\,\% label-flipping adversaries on MNIST, the mechanism maintains 96.7\,\% accuracy, only 0.3 percentage points lower than in a scenario with 30\,\% label-flipping adversaries. This outcome is 51.7 percentage points better than standard FedAvg, which collapses under the same 50\,\% attack. The mechanism is computationally light, budget-bounded, and readily integrates into existing FL frameworks, offering a practical route to economically robust and sustainable FL ecosystems.
\end{abstract}


\section{Introduction}
\IEEEPARstart{F}{ederated} learning (FL) has emerged as a key paradigm for privacy-preserving machine learning, allowing multiple parties to train a shared model without centralizing their raw data \cite{mcmahan_communication-efficient_2017, li_federated_2020}. While promising for sensitive applications such as healthcare \cite{commey_securing_2024}, its distributed and open nature creates a significant vulnerability: data poisoning attacks \cite{bhagoji_analyzing_2019, bagdasaryan_how_2020}. Malicious participants can intentionally submit corrupted model updates to degrade the global model's performance or introduce targeted backdoors \cite{bhagoji_analyzing_2019, bagdasaryan_how_2020}. The breadth and severity of these vulnerabilities are well-documented in recent surveys \cite{lyu_threats_2020}.

The predominant line of defense has been the development of Byzantine-robust aggregation rules. Methods like Krum \cite{blanchard_machine_2017}, Trimmed Mean \cite{yin_byzantine-robust_2018}, and geometric median-based approaches like RFA \cite{pillutla_robust_2022} aim to filter or down-weight malicious updates at the server. However, these methods are fundamentally \emph{reactive}. They often require strong assumptions (e.g., an honest majority), can be computationally intensive, and may discard valuable information from honest clients, thereby slowing convergence. More critically, they fail to address the underlying economic misalignment: honest clients who contribute valuable resources (computation, data, communication) are treated no differently from attackers who seek to sabotage the system.

This economic imbalance threatens the long-term sustainability of open FL ecosystems. If honest participation is not properly incentivized and malicious behavior is not penalized, the system becomes prone to collapse. This leads us to our research question: \emph{Can we design an FL system where participants' economic incentives are aligned with the goal of training a high-quality model, making poisoning attacks unprofitable at equilibrium?}

In this work, we draw upon the principles of mechanism design and game theory to provide an affirmative answer. We propose a proactive, lightweight Bayesian incentive mechanism that shifts the defense from a purely algorithmic problem to a socio-economic one.

\textbf{Contributions.} Our main contributions are as follows:
\begin{itemize}
    \item We formulate the FL training process as a repeated Bayesian game of incomplete information, formally capturing the strategic decisions of clients who can be either benevolent or malicious.
    \item We design a simple yet powerful incentive mechanism where the server uses a small, private validation set to assess the quality of submitted updates. Based on this verification, it issues rewards, effectively creating a market for high-quality model contributions.
    \item We provide formal proofs demonstrating that our mechanism is \textbf{Individually Rational (IR)}, ensuring benevolent clients have a positive expected utility, and \textbf{Incentive Compatible (IC)}, making poisoning an economically dominated strategy for rational attackers.
    \item We conduct extensive experiments on the MNIST and FashionMNIST datasets with non-IID data distributions. Our mechanism demonstrates exceptional robustness, maintaining high accuracy (over 96\% on MNIST and 80\% on FashionMNIST) even when 50\% of clients are malicious—a scenario where standard FedAvg's accuracy catastrophically collapses.
\end{itemize}

\section{Related Work}
\subsection{Robust Aggregation in Federated Learning}
The primary defense against poisoning in FL has centered on Byzantine-robust aggregation. These server-side methods aim to identify and mitigate the impact of malicious updates during the aggregation phase.

\textbf{Federated Averaging (FedAvg)} \cite{mcmahan_communication-efficient_2017} is the standard, non-robust baseline. The server aggregates updates by taking a weighted average of the model parameters from participating clients. While simple and effective in non-adversarial settings, it is highly susceptible to even a single malicious client.

\textbf{Byzantine-Robust Methods} have been developed to counter this vulnerability.
\textbf{Krum} \cite{blanchard_machine_2017} computes a score for each client update based on its sum of squared Euclidean distances to its nearest neighbors and selects only the single update with the lowest score. This is robust but highly inefficient, as it discards the contributions of all other honest clients.
\textbf{Coordinate-wise methods} like Trimmed Mean and Median \cite{yin_byzantine-robust_2018} compute the median or a trimmed mean for each coordinate of the model-weight vectors across all clients. These are robust to extreme values but can be distorted by more subtle attacks.
\textbf{Geometric median-based methods} like RFA \cite{pillutla_robust_2022} compute the geometric median of the client updates, which is more robust to high-dimensional outliers than the arithmetic mean but is computationally expensive.
\textbf{Trusted-source methods} like FLTrust \cite{cao_fltrust_2022} require the server to have a small, clean "root" dataset. The server trains a baseline update on this set and re-weights client updates based on their cosine similarity to this trusted update.
Another line of work uses redundancy and coding theory, such as \textbf{DRACO} \cite{chen_draco_2018}, which uses coded computations to detect and correct errors from stragglers or Byzantine workers, though this often requires significant overhead.

Our work is distinct from these approaches. Instead of relying on statistical properties or a trusted data source, we use a performance-based economic filter that is agnostic to the attack's specific structure.

\subsection{Incentive Mechanisms for FL}
Recognizing the need to motivate participation, researchers have explored economic incentives for FL. These works primarily focus on encouraging high-quality participation from rational, self-interested clients.
Reputation-based systems \cite{kang_incentive_2019} and contract theory \cite{zhan_learning-based_2020} have been proposed to model the contributions of clients and offer tailored rewards. Auction theory has also been applied, for example, in \cite{zhang_incentive_2021}, where the server runs an auction to select clients with the best data quality for a given budget. Some works use blockchain to create decentralized and transparent reward systems, like FedCoin \cite{yang_fedcoin_2020}.

Our work's novelty lies in its direct focus on the security dimension of incentives. We design a mechanism with formal game-theoretic guarantees (IR and IC) to not only encourage honest participation but to actively and provably \emph{discourage} poisoning attacks by making them economically non-viable. The work closest in spirit is perhaps VeriFL \cite{guo_verifl_2021}, which also uses a validation set, but its goal is post-hoc verification and attribution rather than proactive, in-round economic deterrence.

\section{System and Threat Model}

\subsection{System Model}
We consider a standard synchronous FL architecture comprising a central server and a population of $N$ clients. Training proceeds in discrete communication rounds. In each round $t$, the server broadcasts the current global model, $w_t$, to a subset of clients. These clients train the model on their local data and submit their updated model parameters, $w_{i, t+1}$, back to the server. The server then aggregates these updates to produce the next global model, $w_{t+1}$. Our key innovation lies in the verification and payment logic applied before aggregation, as depicted in Figure \ref{fig:system_overview}.

\begin{figure}[!htbp]
    \centering
    \includegraphics[width=\columnwidth]{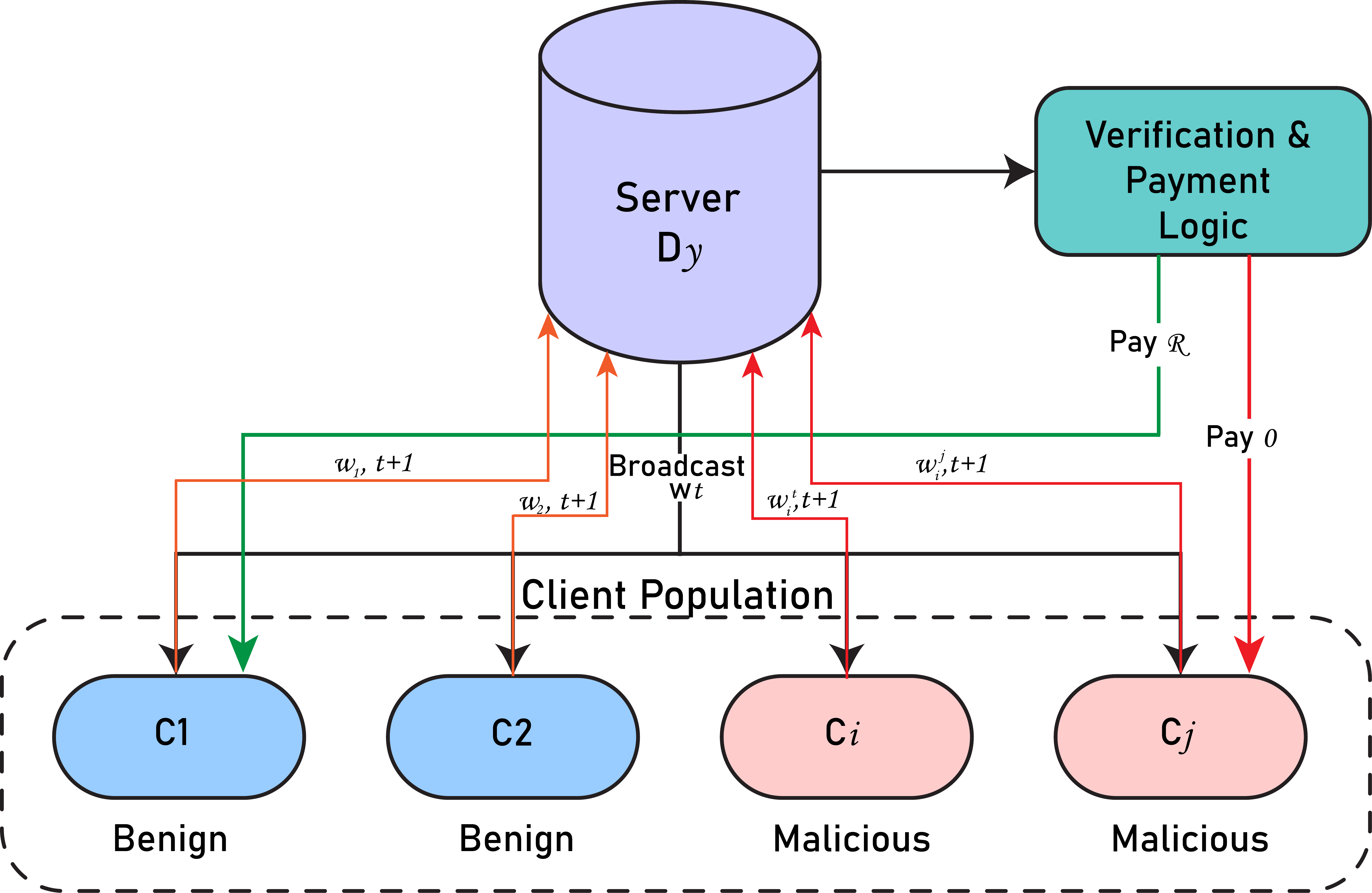}
    \caption{System architecture with the proposed incentive mechanism. The server broadcasts the global model $w_t$. Benign clients (C1, C2) and malicious clients (Ci, Cj) submit their updates. The server uses a private validation set ($D_y$) to assess update quality. High-quality updates are accepted and paid a reward $\mathcal{R}$, while malicious updates are rejected and receive no payment. Only verified updates are used for aggregation.}
    \label{fig:system_overview}
\end{figure}

For our experiments, we use a population of 100 clients. Data from the MNIST and FashionMNIST datasets is partitioned among clients in a non-IID fashion using a Dirichlet distribution with $\alpha=0.5$ to simulate realistic data heterogeneity.

\subsection{Threat Model}
We assume an honest-but-curious server that faithfully executes the protocol but may try to infer information from client updates. A fraction $f$ of the clients are malicious, while the remaining $1-f$ are benevolent. The server has incomplete information; it knows the overall fraction $f$ but does not know the type of any individual client a priori.

The malicious clients aim to degrade the global model's performance on the primary task. To this end, they employ a \textbf{label-flipping} attack, a potent form of data poisoning \cite{bhagoji_analyzing_2019}. During local training, a malicious client maps each true label $y$ to a target label $y'$, effectively training its model on deliberately mislabeled data. For our 10-class datasets, we use the mapping $y' = (y + k) \bmod 10$, where $k$ is an offset (we use $k=1$ for our attacks). This process, illustrated in Figure \ref{fig:threat_model}, forces the client's local model to learn incorrect associations, and when aggregated, these poisoned updates corrupt the global model. More sophisticated attacks, such as targeted backdooring \cite{bagdasaryan_how_2020}, follow a similar principle of manipulating local training to achieve a malicious objective.

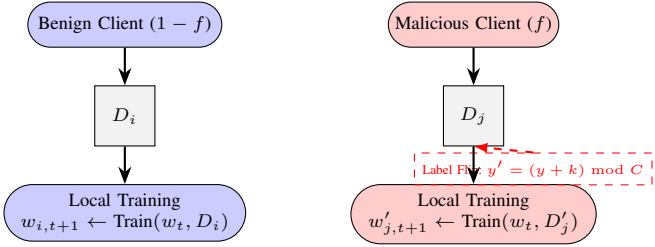
\begin{figure}[!htbp]
    \centering
    \begin{tikzpicture}[
      font=\scriptsize,
      node distance=0.5cm and 0.3cm,
      server/.style={draw, cylinder, shape border rotate=90, minimum height=1.8cm, minimum width=1.2cm, fill=blue!10, align=center},
      client/.style={draw, rounded rectangle, minimum width=1.2cm, minimum height=0.6cm, align=center},
      mal/.style={client, fill=red!10},
      arrow/.style={-Stealth, thick},
      dataflow/.style={arrow, shorten >=1pt, shorten <=1pt}
    ]
    
    \node[client, fill=blue!20, minimum width=2.5cm] (benign) {Benign Client ($1-f$)};
    \node[mal, fill=red!20, right=2.2cm of benign, minimum width=2.5cm] (malicious) {Malicious Client ($f$)};

    % Data nodes
    \node[draw, rectangle, fill=gray!10, minimum size=0.8cm, below=0.5cm of benign] (data_b) {$D_i$};
    \node[draw, rectangle, fill=gray!10, minimum size=0.8cm, below=0.5cm of malicious] (data_m) {$D_j$};

    % Training process nodes
    \node[client, fill=blue!20, below=0.5cm of data_b, minimum width=2.5cm] (train_b) {Local Training\\$w_{i,t+1} \leftarrow \text{Train}(w_t, D_i)$};
    \node[mal, fill=red!20, below=0.5cm of data_m, minimum width=2.5cm] (train_m) {Local Training\\$w'_{j,t+1} \leftarrow \text{Train}(w_t, D'_j)$};

    % Label flipping annotation
    \node[draw, dashed, red, inner sep=3pt, below right=0.08cm and -1.2cm of data_m, font=\tiny] (flip_note) {Label Flip: $y'=(y+k)\bmod C$};
    \draw[-latex, red, dashed, thick] (flip_note.north) -- (data_m.south);

    % Arrows
    \draw[arrow] (benign) -- (data_b);
    \draw[arrow] (malicious) -- (data_m);
    \draw[arrow] (data_b) -- (train_b);
    \draw[arrow] (data_m) -- (train_m);
    \end{tikzpicture}
    \caption{The threat model. A fraction $f$ of clients are malicious. They poison their local dataset $D_j$ by flipping labels to create $D'_j$ before training their local model. Benign clients train on their original, clean data $D_i$.}
    \label{fig:threat_model}
\end{figure}

\section{Bayesian Incentive Mechanism}

\subsection{Game Formulation}
We model each round of federated training as a Bayesian game of incomplete information, defined by the tuple $\Gamma = \langle \mathcal{N}, \{\Theta_i\}, \{A_i\}, \{u_i\}, p \rangle$:
\begin{itemize}
    \item $\mathcal{N}$: The set of players, consisting of the server and $N$ clients.
    \item $\Theta_i$: The set of types for each client $i$, $\Theta_i = \{\text{benevolent}, \text{malicious}\}$. A client's type is private information.
    \item $A_i$: The action space for client $i$. A client chooses an action $a_i \in \{\text{honest\_update}, \text{poisoned\_update}\}$.
    \item $p(\theta_i)$: The server's prior belief about the probability that client $i$ is of a certain type. We assume a common prior where $P(\theta_i=\text{malicious}) = f$.
    \item $u_i$: The utility function for client $i$. The utility is determined by the payment received from the server, $p_i$, minus the operational cost incurred, $C_i$. Thus, $u_i = p_i - C_i$. We assume a uniform cost $C$ for all clients for simplicity.
\end{itemize}

\subsection{The Verification and Payment Mechanism}
The core of our defense is a direct revelation mechanism where the server incentivizes clients to reveal their "true" contribution quality.

\begin{definition}[Verification Mechanism]
The server holds a small, private, and clean validation dataset, $D_y$. Upon receiving a model update $w_{i, t+1}$ from client $i$, the server evaluates its loss on this set: $L_i = \mathcal{L}(w_{i, t+1}; D_y)$. The update is considered "verified" if its loss is below a predefined quality threshold $\tau$. The payment rule is:
\begin{equation}
p_i(w_{i, t+1}) = \begin{cases} 
\mathcal{R} & \text{if } \mathcal{L}(w_{i, t+1}; D_y) < \tau \\
0 & \text{otherwise}
\end{cases}
\label{eq:payment_rule}
\end{equation}
where $\mathcal{R}$ is a fixed reward for a verified update. Only the set of verified updates, $\mathcal{V}_t$, are aggregated to form the next global model:
\begin{equation}
w_{t+1} = \frac{1}{|\mathcal{V}_t|} \sum_{w_j \in \mathcal{V}_t} w_j
\end{equation}
\end{definition}

This mechanism is lightweight, as it only requires a single forward pass on a small validation set for each client, a negligible cost compared to the training itself. The logic is outlined in Algorithm \ref{alg:mechanism}.

\subsection{Mechanism Properties}
A robust mechanism must make it profitable for honest players to participate and unprofitable for malicious players to attack. These correspond to the game-theoretic properties of Individual Rationality (IR) and Incentive Compatibility (IC).

\begin{theorem}[Individual Rationality]
For a benevolent client choosing the honest action, the expected utility is positive if the reward $\mathcal{R}$ and cost $C$ are set such that $\mathcal{R} > C / P(\text{verification} | \text{honest})$.
\end{theorem}
\begin{proof}
A benevolent client's action is to submit an honestly trained update. Let this action be $a_h$. The client's update will be verified if its loss on $D_y$ is less than $\tau$. Let the probability of this event be $P_v^h = P(\mathcal{L}(w_{honest}) < \tau)$. The expected utility for a benevolent client is:
\begin{equation}
\mathbb{E}[u_i | \theta_i=\text{benevolent}, a_i=a_h] = P_v^h \cdot \mathcal{R} + (1 - P_v^h) \cdot 0 - C
\end{equation}
For participation to be rational, this expected utility must be greater than 0 (the utility of not participating).
\begin{equation}
P_v^h \cdot \mathcal{R} - C > 0 \implies \mathcal{R} > \frac{C}{P_v^h}
\end{equation}
Since an honest update is designed to minimize the loss, its loss on a clean validation set will be low. Therefore, for a reasonably set $\tau$, $P_v^h \approx 1$. With our parameters $\mathcal{R}=10$ and $C=2$, the condition becomes $10 > 2/P_v^h$, which holds easily. Thus, honest participation is economically rational.
\end{proof}

\begin{theorem}[Incentive Compatibility]
For a rational, self-interested client, choosing a poisoned action is an economically dominated strategy if the attack significantly increases the model's loss on a clean validation set.
\end{theorem}
\begin{proof}
A malicious client's goal is to submit a poisoned update, $a_p$, to degrade the model. This action inherently increases the model's true loss. Let the probability of a poisoned update passing verification be $P_v^m = P(\mathcal{L}(w_{poisoned}) < \tau)$. The expected utility for taking the poisoned action is:
\begin{equation}
\mathbb{E}[u_i | a_i=a_p] = P_v^m \cdot \mathcal{R} - C
\end{equation}
For a label-flipping attack, the resulting model will perform poorly on the correctly labeled validation set $D_y$. Thus, its loss will be high, and for a reasonable $\tau$, the probability of verification will be near zero, $P_v^m \approx 0$. The expected utility becomes:
\begin{equation}
\mathbb{E}[u_i | a_i=a_p] \approx 0 \cdot \mathcal{R} - C = -C
\end{equation}
A rational agent will compare this negative utility to the utility of not participating (utility 0) or participating honestly (positive utility, from Theorem 1). Since $-C < 0$, the poisoning strategy is strictly dominated by non-participation. This demonstrates that the mechanism is incentive-compatible, as it disincentivizes the malicious action.
\end{proof}

\section{Experimental Setup}
\label{sec:setup}
We implemented our system in PyTorch and conducted experiments to evaluate its performance against standard baselines.

\begin{itemize}
    \item \textbf{Datasets:} We used two benchmark datasets: \textbf{MNIST} (handwritten digits) and \textbf{FashionMNIST} (apparel images). Both have 10 classes. Data was distributed among 100 clients using a Dirichlet distribution ($\alpha=0.5$) to simulate a non-IID environment.
    \item \textbf{Model:} A Convolutional Neural Network (CNN) with two convolutional layers (32 and 64 filters, 5x5 kernel), each followed by max-pooling, and two fully-connected layers (1024 units and 10 units for the output).
    \item \textbf{Training Protocol:} 40 communication rounds, 3 local epochs per round, batch size 32, and SGD with a learning rate of 0.01.
    \item \textbf{Attack Scenarios:} We tested with malicious client fractions ($f$) of 30\%, 40\%, and 50\%. Malicious clients performed a label-flipping attack ($y' = (y+1) \bmod 10$).
    \item \textbf{Baselines for Comparison:}
        \begin{itemize}
            \item \textbf{FedAvg:} The standard, non-robust federated averaging algorithm.
            \item \textbf{Krum:} A well-known Byzantine-robust aggregation rule that selects the single "best" update.
        \end{itemize}
    \item \textbf{Mechanism Parameters:} Reward $\mathcal{R}=10$, Cost $C=2$, and verification threshold $\tau=2.5$. The server's private validation set $D_y$ contained 200 randomly sampled examples.
\end{itemize}

The pseudocode for the server's logic in our proposed mechanism is detailed in Algorithm \ref{alg:mechanism}.

\begin{algorithm}[!htbp]
\caption{Federated Learning with Bayesian Incentive Mechanism (Server-Side Logic)}
\label{alg:mechanism}
\begin{algorithmic}[1]
\REQUIRE Reward $\mathcal{R}$, Cost $C$, Threshold $\tau$, Validation set $D_y$
\STATE Initialize global model $w_0$
\FOR{each communication round $t = 0, 1, \dots, T-1$}
    \STATE Broadcast global model $w_t$ to all clients
    \STATE $\mathcal{U}_t \leftarrow \emptyset$ \COMMENT{Collect all incoming updates}
    \FORALL{clients $i \in \{1, \dots, N\}$ \textbf{in parallel}}
        \STATE $w_{i, t+1} \leftarrow$ \textsc{ClientUpdate}$(w_t, D_i)$
        \STATE $\mathcal{U}_t \leftarrow \mathcal{U}_t \cup \{w_{i, t+1}\}$
    \ENDFOR
    
    \STATE $\mathcal{V}_t \leftarrow \emptyset$ \COMMENT{Set of verified updates}
    \FOR{each update $w_i \in \mathcal{U}_t$}
        \STATE $L_i \leftarrow$ \textsc{EvaluateLoss}$(w_i, D_y)$
        \IF{$L_i < \tau$}
            \STATE Pay $\mathcal{R}$ to client $i$
            \STATE $\mathcal{V}_t \leftarrow \mathcal{V}_t \cup \{w_i\}$
        \ELSE
        	\STATE Pay $0$ to client $i$ \COMMENT{Client incurs cost C}
        \ENDIF
    \ENDFOR
    
    \IF{$|\mathcal{V}_t| > 0$}
        \STATE $w_{t+1} \leftarrow \frac{1}{|\mathcal{V}_t|} \sum_{w \in \mathcal{V}_t} w$ \COMMENT{Aggregate verified updates}
    \ELSE
        \STATE $w_{t+1} \leftarrow w_t$ \COMMENT{No updates verified, maintain model}
    \ENDIF
\ENDFOR
\end{algorithmic}
\end{algorithm}

\section{Experimental Results}
Our experiments provide a comprehensive view of the mechanism's performance, robustness, and economic effects across both MNIST and FashionMNIST datasets. We compare our proposed Bayesian Incentive Mechanism (which we will refer to as "Mechanism") against standard FedAvg and the well-known robust aggregation rule, Krum.

\subsection{Overall Performance and Robustness}
We first present a high-level analysis of the mechanism's resilience to an increasing number of attackers. Figure \ref{fig:comprehensive_analysis} summarizes the key outcomes. The top-left panel shows the final test accuracy on MNIST as the fraction of malicious clients increases from 30\% to 50\%. While FedAvg's accuracy plummets from 95.3\% to 43.5\%, and Krum's performance degrades, our mechanism's accuracy remains exceptionally stable above 96.7\%. The top-right panel reinforces this, showing our mechanism achieving 97.0\% accuracy on MNIST and 80.3\% on FashionMNIST in the challenging 40\% malicious scenario, significantly outperforming both baselines.

The robustness of our approach is quantified in Table \ref{tab:robustness_analysis}. When the attacker population grows from 30\% to 50\%, FedAvg's accuracy suffers a catastrophic drop of 51.8 percentage points on MNIST and 45.3 points on FashionMNIST. Krum also proves vulnerable, especially on FashionMNIST, where its accuracy collapses. In stark contrast, our mechanism's performance degrades by a negligible 0.24 points on MNIST and 1.22 points on FashionMNIST, demonstrating that its economic filtering effectively insulates the global model from the number of attackers.

Finally, the bottom-right panel of Figure \ref{fig:comprehensive_analysis} validates our economic model. The average utility for an honest client quickly converges to the theoretical maximum of $\mathcal{R} - C = 8$, confirming that honest participation is consistently and profitably rewarded.

\begin{figure}[!htbp]
    \centering
    \includegraphics[width=\linewidth]{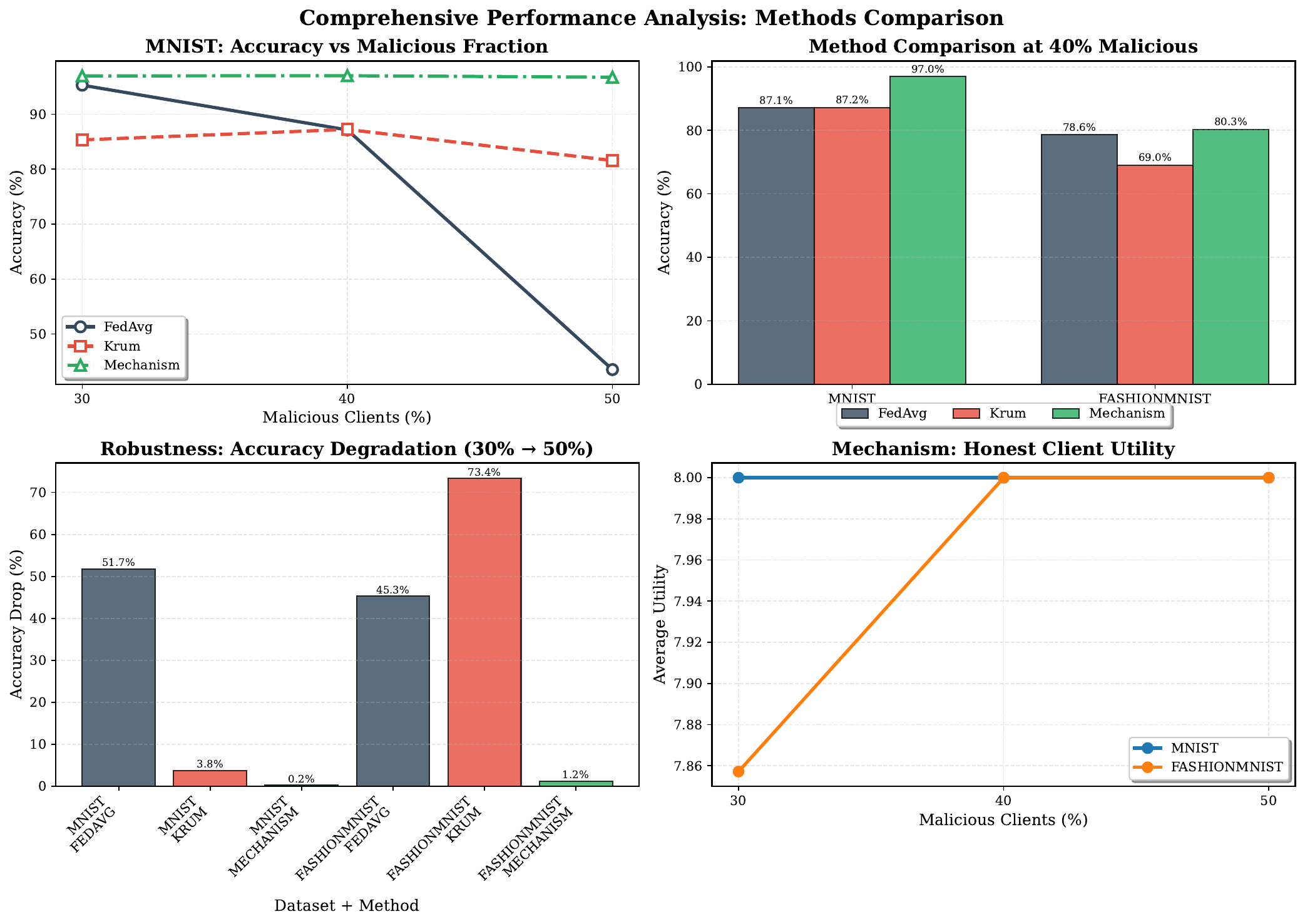} 
    \caption{Comprehensive performance analysis comparing FedAvg, Krum, and our proposed mechanism. (Top-Left) Final test accuracy on MNIST vs. the fraction of malicious clients. (Top-Right) A direct comparison of final accuracy at 40\% malicious for both MNIST and FashionMNIST. (Bottom-Left) The drop in accuracy when increasing malicious clients from 30\% to 50\%, highlighting robustness. (Bottom-Right) The average utility for honest clients participating in our mechanism, showing convergence to the theoretical maximum of 8.}
    \label{fig:comprehensive_analysis}
\end{figure}

\begin{table}[!htbp]
\centering
\caption{Robustness Analysis: Accuracy Degradation when Increasing Malicious Client Fraction from 30\% to 50\%.}
\label{tab:robustness_analysis}
\resizebox{\columnwidth}{!}{%
\begin{tabular}{@{}lcccc@{}}
\toprule
\textbf{Method} & \textbf{Acc. @ 30\%} & \textbf{Acc. @ 50\%} & \textbf{Degrad. (\% pts)} \\
\midrule
\multicolumn{4}{c}{\textbf{MNIST Dataset}} \\
\midrule
FedAvg & 95.27\% & 43.52\% & \textcolor{red}{51.75} \\
Krum & 85.31\% & 81.56\% & 3.75 \\
\textbf{Mechanism} & \textbf{96.96\%} & \textbf{96.72\%} & \textbf{\textcolor{blue}{0.24}} \\
\midrule
\multicolumn{4}{c}{\textbf{FashionMNIST Dataset}} \\
\midrule
FedAvg & 80.74\% & 35.44\% & \textcolor{red}{45.30} \\
Krum & 73.68\% & 0.33\% & \textcolor{red}{73.35} \\
\textbf{Mechanism} & \textbf{81.89\%} & \textbf{80.67\%} & \textbf{\textcolor{blue}{1.22}} \\
\bottomrule
\end{tabular}%
}
\end{table}

\subsection{Detailed Analysis on MNIST}
Figure \ref{fig:mnist_analysis} provides a detailed view of the training dynamics on MNIST. The final accuracy (top-left panel) remains consistently high for our mechanism, achieving 96.7\% even with 50\% attackers, as also detailed in Table \ref{tab:detailed_results}. This stability is a direct result of the economic filter. The training progress (bottom-right panel) for the 40\% attack scenario shows our mechanism's smooth and rapid convergence, while FedAvg is erratic and Krum is noisy.

From an economic perspective, the top-right panel shows that the average utility for honest clients is stable and positive, rapidly converging towards the ideal payoff of 8.0. This empirically validates Theorem 1 (Individual Rationality). Malicious clients, whose updates are consistently rejected, receive zero payment and incur the cost $C$, yielding a negative utility. This validates Theorem 2 (Incentive Compatibility), as attacking is an economically irrational choice. The total server expenditure (bottom-left panel) remains bounded and stable, demonstrating the economic sustainability of the system.

\begin{figure}[!htbp]
    \centering
    \includegraphics[width=\linewidth]{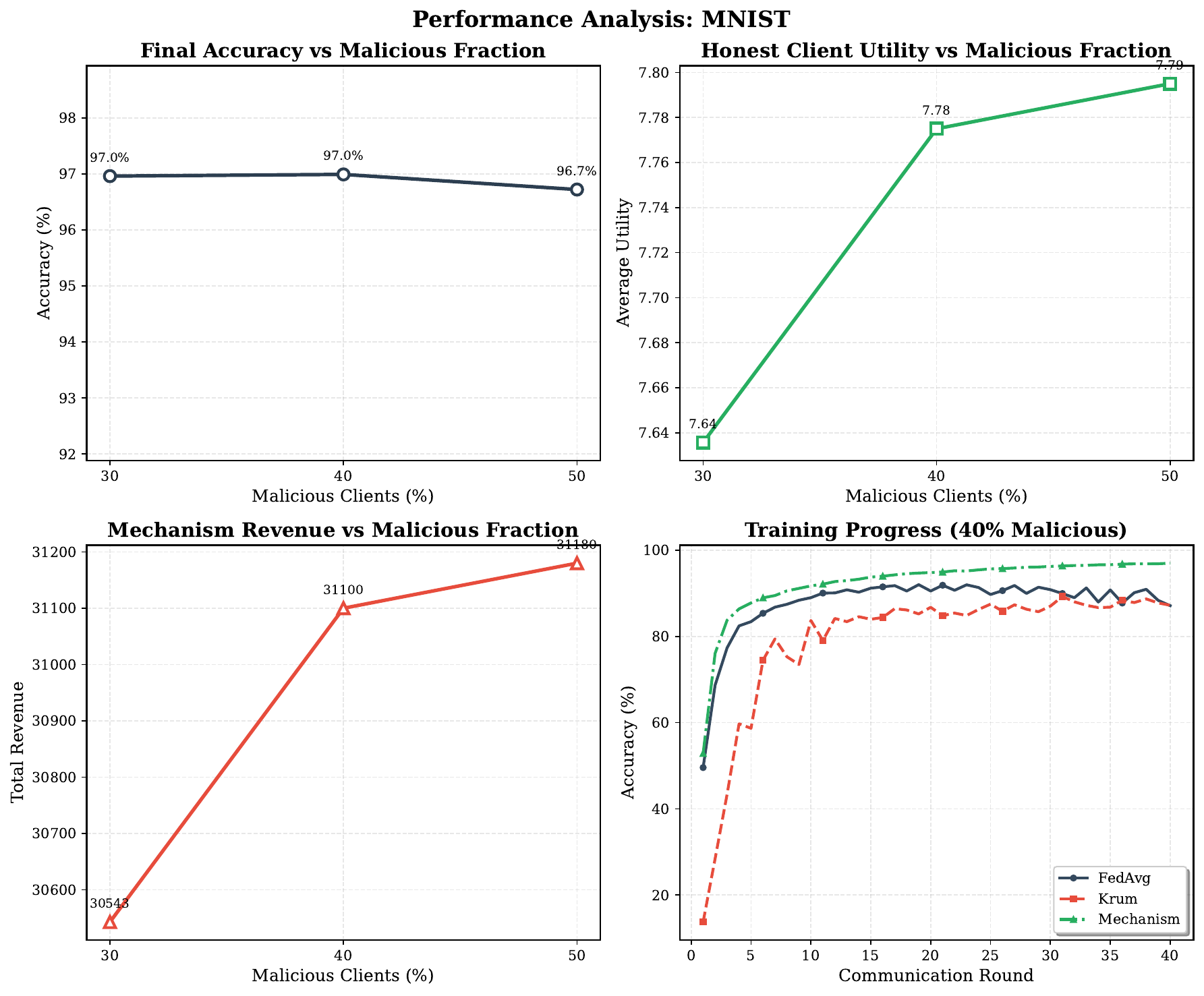} 
    \caption{Detailed performance and economic analysis on the MNIST dataset. (Top-Left) Final accuracy remains high and stable. (Top-Right) Average utility for honest clients is stable and positive. (Bottom-Left) Total server revenue remains bounded. (Bottom-Right) Training progress at 40\% malicious shows superior convergence.}
    \label{fig:mnist_analysis}
\end{figure}

\subsection{Detailed Analysis on FashionMNIST}
The results on the more challenging FashionMNIST dataset (Figure \ref{fig:fashionmnist_analysis}) further underscore our mechanism's strengths. While our mechanism maintains over 80\% accuracy across all attack levels, FedAvg's performance degrades sharply, and Krum fails catastrophically, with its accuracy dropping to near-random levels (0.33\% with 50\% attackers), highlighting the fragility of distance-based metrics on more complex tasks. These final accuracy numbers are compiled in Table \ref{tab:detailed_results}.

The economic outcomes are equally strong. The utility for honest clients (top-right panel) remains positive, ensuring participation is viable. The training curve at 40\% malicious (bottom-right panel) again confirms our mechanism's stability and superior convergence. The total server revenue (bottom-left panel) is stable, showing the mechanism is not just robust but also budget-conscious, as it avoids paying for low-quality or malicious contributions.

\begin{figure}[!htbp]
    \centering
    \includegraphics[width=\linewidth]{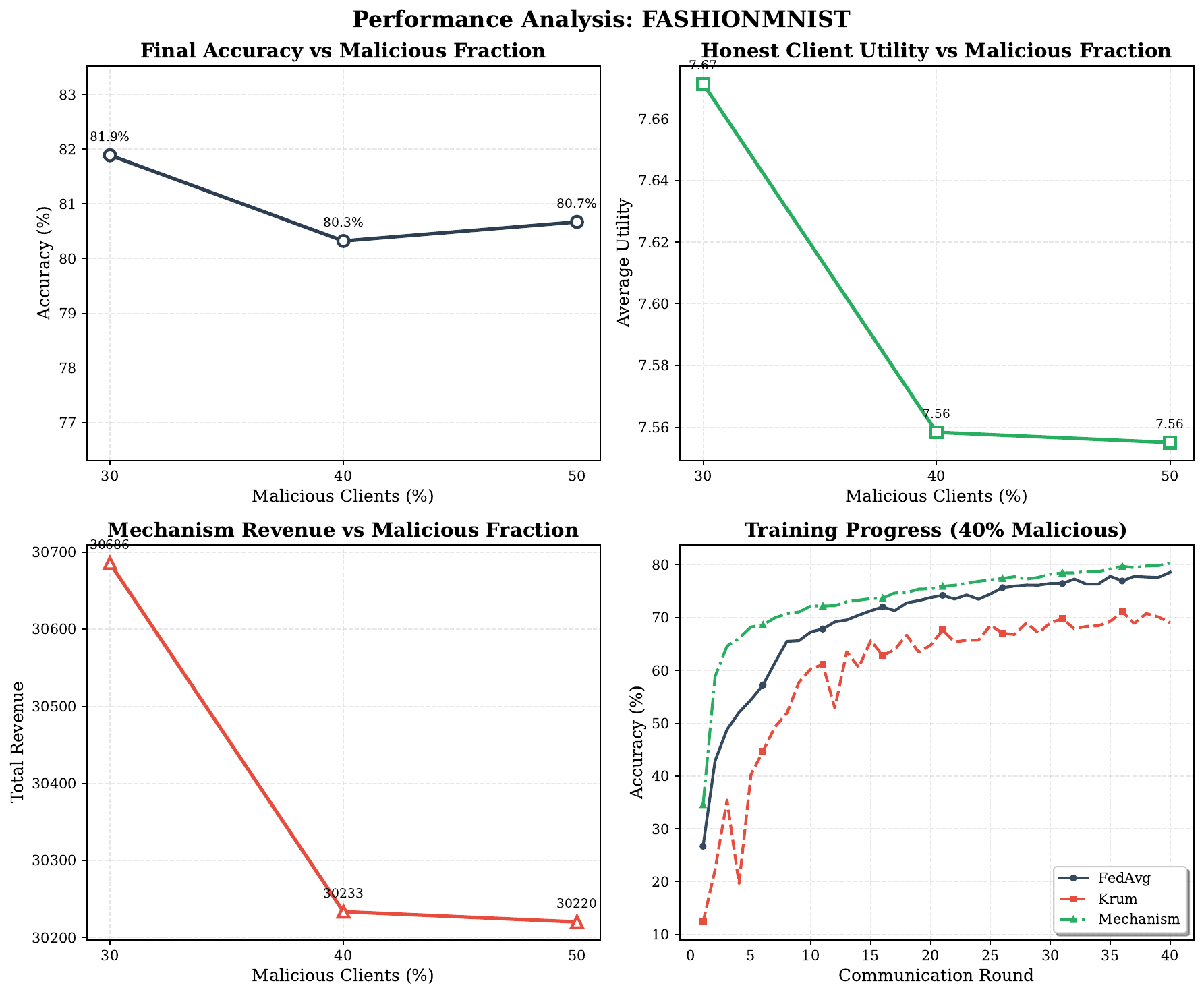} 
    \caption{Detailed performance and economic analysis on the FashionMNIST dataset. (Top-Left) Our mechanism maintains high accuracy while baselines fail. (Top-Right) Average utility for honest clients remains positive and viable. (Bottom-Left) Total server revenue shows controlled expenditure. (Bottom-Right) Training progress confirms the stability and effectiveness of our mechanism.}
    \label{fig:fashionmnist_analysis}
\end{figure}

\begin{table}[!htbp]
\centering
\caption{Detailed Final Performance and Economic Metrics.}
\label{tab:detailed_results}
\resizebox{\columnwidth}{!}{%
\begin{tabular}{@{}lllccc@{}}
\toprule
\textbf{Dataset} & \textbf{Malicious \%} & \textbf{Method} & \textbf{Final Acc. \%} & \textbf{Honest Utility} & \textbf{Total Revenue} \\
\midrule
\multirow{9}{*}{MNIST} & \multirow{3}{*}{30\%} & FedAvg & 95.27 & --- & --- \\
 & & Krum & 85.31 & --- & --- \\
 & & \textbf{Mechanism} & \textbf{96.96} & \textbf{7.64} & \textbf{30.5k} \\
 \cmidrule(lr){2-6}
 & \multirow{3}{*}{40\%} & FedAvg & 87.13 & --- & --- \\
 & & Krum & 87.22 & --- & --- \\
 & & \textbf{Mechanism} & \textbf{97.00} & \textbf{7.78} & \textbf{31.1k} \\
 \cmidrule(lr){2-6}
 & \multirow{3}{*}{50\%} & FedAvg & 43.52 & --- & --- \\
 & & Krum & 81.56 & --- & --- \\
 & & \textbf{Mechanism} & \textbf{96.70} & \textbf{7.79} & \textbf{31.2k} \\
\midrule
\multirow{9}{*}{FMNIST} & \multirow{3}{*}{30\%} & FedAvg & 80.74 & --- & --- \\
 & & Krum & 73.68 & --- & --- \\
 & & \textbf{Mechanism} & \textbf{81.90} & \textbf{7.67} & \textbf{30.7k} \\
 \cmidrule(lr){2-6}
 & \multirow{3}{*}{40\%} & FedAvg & 78.60 & --- & --- \\
 & & Krum & 69.03 & --- & --- \\
 & & \textbf{Mechanism} & \textbf{80.30} & \textbf{7.56} & \textbf{30.2k} \\
 \cmidrule(lr){2-6}
 & \multirow{3}{*}{50\%} & FedAvg & 35.44 & --- & --- \\
 & & Krum & 0.33 & --- & --- \\
 & & \textbf{Mechanism} & \textbf{80.70} & \textbf{7.56} & \textbf{30.2k} \\
\bottomrule
\end{tabular}%
}
\end{table}

\section{Conclusion}
In this paper, we introduced a game-theoretic incentive mechanism that provides a proactive, economic defense against data poisoning attacks in federated learning. By framing the FL process as a Bayesian game and implementing a simple, low-cost verification step, our mechanism successfully aligns client incentives with the global objective of training an accurate model.

Our extensive experiments on MNIST and FashionMNIST demonstrate the mechanism's remarkable effectiveness. It maintains high accuracy and stability even under extreme attack conditions (50\% malicious clients) where standard FedAvg fails completely and the popular robust aggregator Krum struggles or fails. Our robustness analysis (Table \ref{tab:robustness_analysis}) quantifies this resilience, showing negligible performance degradation under increased attack intensity. We formally proved, and empirically validated through detailed results (Figures \ref{fig:mnist_analysis} and \ref{fig:fashionmnist_analysis}, and Table \ref{tab:detailed_results}), that the mechanism is individually rational for honest participants and incentive-compatible for deterring attackers by making malicious behavior economically non-viable.

This work shows that shifting focus from purely algorithmic defenses to socio-economic ones is a powerful and practical strategy for building secure, robust, and sustainable federated learning systems. Future work could explore adaptive verification thresholds, extend the mechanism to defend against more subtle attack strategies like model backdooring, and investigate its application in fully decentralized settings.

\bibliographystyle{IEEEtran}
\bibliography{references}

\end{document}